\documentclass[10pt]{article}
\usepackage[letterpaper,left=0.78in,right=0.78in,top=0.55in,bottom=0.7in,headheight=12pt,headsep=12pt]{geometry}
\usepackage[T1]{fontenc}
\usepackage[utf8]{inputenc}
\usepackage{lmodern,microtype,amsmath,amssymb,amsthm,mathtools}
\usepackage{booktabs,tabularx,array,fancyhdr,titlesec,xcolor,graphicx}
\usepackage{tikz}
\usetikzlibrary{arrows.meta,positioning}
\usepackage[hidelinks]{hyperref}
\definecolor{navy}{RGB}{25,58,82}
\definecolor{warm}{RGB}{252,246,239}
\definecolor{pale}{RGB}{246,248,250}
\titleformat{\section}{\large\bfseries\color{navy}}{\thesection}{0.65em}{}
\titleformat{\subsection}{\normalsize\bfseries\color{navy}}{\thesubsection}{0.65em}{}
\titlespacing*{\section}{0pt}{12pt}{5pt}
\titlespacing*{\subsection}{0pt}{8pt}{3pt}
\newtheorem{theorem}{Theorem}
\newcommand{\op}[1]{\textsf{#1}}
\newcommand{\fv}{\operatorname{fv}}
\newcolumntype{Y}{>{\raggedright\arraybackslash}X}
\newcommand{\callout}[1]{\begin{center}\fcolorbox{navy}{warm}{\begin{minipage}{0.94\linewidth}#1\end{minipage}}\end{center}}
\hypersetup{pdftitle={From Question to Evidence: A Small Analytical Algebra for Governed Data Analysis},pdfauthor={Matt Bray; Logan Green; Hemant Jomraj; Shardul Pande; Viktoria Rojkova; MasterControl AI Lab},pdfsubject={Expressiveness, deterministic policy execution, and a controlled comparison with runtime tool planning},pdfkeywords={relational algebra, governed analytics, deterministic execution, tool-using agents, SQL agents, evidence, policy}}
\title{\vspace{-0.8em}\textbf{From Question to Evidence: A Small Analytical Algebra for Governed Data Analysis}\\[3pt]\large Expressiveness, Deterministic Policy Execution, and a Controlled Comparison with Runtime Tool Planning}
\author{Matt Bray, Logan Green, Hemant Jomraj, Shardul Pande, Viktoria Rojkova\\[2pt]\large MasterControl AI Lab}
\date{September 2026}
\begin{document}
\maketitle
\vspace{-3.2em}
\begin{abstract}
A factual analyzer should not invent a new measuring method each time it answers a question. The difficulty is not generating a number; it is preserving what the number means: the population, time interval, measure, definitions, and supporting records. We study a simple architecture for bounded, read-only enterprise analytics. A language model may interpret the user's wording into governed meaning, but a deterministic policy selects a pre-written analytical program. The program is built from a small set of typed operations and returns both the result and its evidence.

The technical question is whether such restriction sacrifices analytical power. Starting from finite relations and first-order satisfaction, we give a constructive translation into relational operations. We then add explicit aggregation, comparison, windows, ranking, and versioned similarity to cover a stated analytical class. The result is scoped completeness: every specification in that class has an exact finite program, although the language is not a general-purpose programming system and a deployed policy catalog need not cover every possible question. We also prove replay: fixed governed meaning, policy, data, kernel state, numerical rules, and output contract imply the same result and evidence.

We compare two ways of answering the same analytical questions across 440 runs. In the first, three independent 8B open-source models decide how to query the data, what tools to use, and generate the SQL procedure at runtime. In the second, Qwen3-8B only interprets what the user is asking; after that, a deterministic policy selects and executes the pre-approved analytical program corresponding to that interpretation. Across 330 runtime-planning episodes, no final procedure matched the full answer-and-evidence contract on the development snapshot and four held-out variants. The policy-executed analyzer matched the contract in 110 of 110 episodes. In this specific setup, runtime tool planning did not perform successfully on the primary metric. The result is empirical and configuration-specific; it is not a proof that agents cannot succeed under other search, verification, model, or tool designs.
\end{abstract}

\section{The problem: 17 is not just a number}
Consider a question that should be almost boring:
\callout{\centering\textbf{How many approved temperature-excursion deviations occurred last week?}}
Suppose the answer is 17. Producing 17 once is easy. Preserving what 17 means is harder. ``Last week'' could mean a calendar week or trailing seven days. ``Temperature excursion'' could mean a controlled category or semantic similarity in free text. A join can count deviation records or observation rows. Voided records may or may not qualify. Two systems can print the same number while measuring different populations.

For governed analysis, we therefore treat the answer as a structured contract rather than a scalar:
\begin{equation}R=(m,v,E),\label{eq:contract}\end{equation}
where \(m\) is the accepted governed meaning, \(v\) is the analytical value or table, and \(E\) is the role-labeled supporting evidence. An execution certificate separately records the data snapshot, policy version, program version, kernel versions, and serialization rules.

This leads to a simple architectural boundary:
\begin{center}\textbf{The model may interpret the question. It does not choose the analytical method.}\end{center}
The distinction matters because an agent that is free to choose tools, joins, filters, sequence, and stopping rule is not merely calculating an answer. It is designing an analytical procedure at request time.

\begin{figure}[t]
\centering
\begin{tikzpicture}[node distance=3mm,box/.style={draw=navy,fill=pale,rounded corners=2pt,align=center,text width=2cm,minimum height=0.66cm,font=\small},>=Stealth]
\node[box] (a) {Natural-\\language\\question};
\node[box,right=of a] (b) {Semantic\\interpretation};
\node[box,right=of b] (c) {Canonical\\request};
\node[box,right=of c] (d) {Deterministic\\policy};
\node[box,right=of d] (e) {Pre-written\\program};
\node[box,right=of e] (f) {Result +\\evidence};
\foreach \x/\y in {a/b,b/c,c/d,d/e,e/f} \draw[->,gray] (\x)--(\y);
\node[below=4mm of b,font=\scriptsize,text=gray] {probabilistic boundary};
\node[below=4mm of d,font=\scriptsize,text=teal] {governed analytical core};
\end{tikzpicture}
\caption{The proposed boundary. Language remains flexible at the interface; the analytical method is fixed by policy and versioned code.}\label{fig:architecture}
\end{figure}
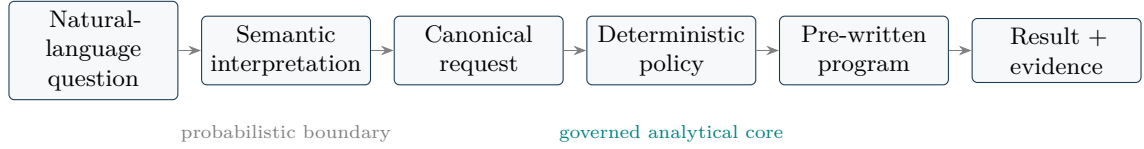

\section{A small analytical language}
The proposed execution language has eight practical families. The names are less important than their contracts: each operation has typed inputs, typed outputs, declared parameters, error behavior, and evidence semantics.
\begin{center}\small
\begin{tabularx}{\linewidth}{@{}p{1.0in}p{2.35in}Y@{}}
\toprule
\textbf{Family}&\textbf{What it fixes}&\textbf{Typical question}\\\midrule
\op{RESTRICT}&Which records qualify?&Temperature excursions from the approved population.\\
\op{SHAPE}&Which fields and deterministic derivations are returned?&Express duration in hours; retain source identifiers.\\
\op{RELATE}&How are governed relations combined?&Link deviations to sites, batches, equipment, or CAPAs.\\
\op{AGGREGATE}&At what grain is a measure calculated?&Count distinct deviations by site.\\
\op{COMPARE}&What is the baseline or denominator?&Compare a site rate with the network baseline.\\
\op{WINDOW}&Which time or ordered frame applies?&Compare this week with the previous week.\\
\op{RANK}&What total order and tie rule apply?&Return the top contributing sites.\\
\op{SIMILAR}&Which versioned semantic or numeric neighborhood applies?&Group free-text causes by a fixed similarity rule.\\
\bottomrule
\end{tabularx}\end{center}
The running count can be implemented by a reviewed program such as
\begin{center}\footnotesize\ttfamily
RESTRICT[approved population] \(\to\) WINDOW[previous calendar week] \(\to\)\\
RESTRICT[temperature-excursion definition] \(\to\) AGGREGATE[count distinct deviation\_id] \(\to\) SHAPE[result + evidence].
\end{center}
This is not a plan generated by the model. It is a program authored and tested before deployment, or deterministically expanded from approved policy rules. A different supported question maps to a different approved program.

The practical question is whether this small language can be broad enough. Database theory gives the right analogy: relational algebra is restrictive, yet it is complete for an important class of database queries. ``Complete'' here never means universal computation; it means complete relative to a stated query class.

\section{From relational logic to analytical programs}
\subsection{Relational foundation}
Let a finite database \(D\) assign a finite set of tuples to each relation in a typed schema. For each sort \(s\), let \(A_s\) be a declared finite candidate domain. A relational question is specified independently of our primitives by a first-order formula \(\varphi\):
\begin{equation}q_\varphi(D)=\{\bar a\in A_{s_1}\times\cdots\times A_{s_k}:D\models\varphi(\bar a)\}.\label{eq:spec}\end{equation}
The formula may use relational atoms, equality and declared comparisons, Boolean connectives, and quantification over the finite domains. The definition says what the answer is; it says nothing about how to execute it.

The relational core uses selection \(\sigma\), projection \(\pi\), renaming \(\rho\), product \(\times\), union \(\cup\), and difference \(\setminus\). These operations correspond naturally to logical constructions. Selection implements row-local predicates. Projection removes variables and therefore implements existential quantification. Union implements disjunction. Difference from a declared candidate universe implements negation. Universal quantification is implemented by removing assignments that have a counterexample.
\begin{theorem}[Constructive relational expressiveness]\label{thm:relational}
For every finite-domain first-order specification \(q_\varphi\), there is a finite program built from \op{RESTRICT}, \op{SHAPE}, and \op{RELATE} whose result equals \(q_\varphi(D)\) for every admissible database \(D\). Conversely, every program in this relational core has an equivalent finite-domain first-order specification.
\end{theorem}
\textit{Proof idea.} For a set of variables \(X\), form the finite assignment relation \(U_X\) from their declared domains. Translate each formula recursively: atoms become selections over renamed source relations, negation becomes \(U_X\setminus T_X(\varphi)\), disjunction becomes union, and existential quantification becomes projection. Universal quantification removes assignments having a counterexample. Structural induction shows that a tuple belongs to the translated relation exactly when it satisfies the original formula. The reverse direction follows by induction over relational expressions. The full construction is in Appendix~\ref{app:proofs}.

Two boundaries are important. First, difference or an equivalent non-monotone operation is needed for absence questions such as ``sites with no deviation.'' Second, pure relational operations do not generate a variable-size count as a new scalar. Counting therefore requires an explicit analytical extension.

\subsection{Analytical extensions}
Let \(K\) be a finite registry of declared analytical kernels. In addition to relational stages, a specification may: apply a registered scalar derivation; partition a population and aggregate; construct a time/order window; compare aligned measures; rank under a total order; or apply a fixed proximity function or immutable score table. These constructions have mathematical definitions independent of the eight family names.
\begin{theorem}[Completeness for the stated analytical class]\label{thm:analytical}
If the registry implements the declared constructions with their stated semantics, every finite acyclic specification in this analytical class has a finite well-typed program over the eight families that returns the same value, evidence, or declared domain error on every admissible input.
\end{theorem}
\textit{Proof idea.} Translate each relational stage by Theorem~\ref{thm:relational}. Implement scalar derivations by \op{SHAPE}, grouping and folds by \op{AGGREGATE}, time/order frames by \op{WINDOW}, ranking by \op{RANK}, and fixed proximity by \op{SIMILAR}. Process stages in dependency order. If earlier stages equal their specifications, the next declared kernel receives the same inputs and returns the same output. Induction completes the construction.

The theorem is deliberately scoped. It does not say that the registry contains every useful statistic, causal model, scientific method, or business question. It says something more checkable: once a question is defined inside the stated class, a finite exact program exists.

\section{Policy owns the method}
Expressiveness and deployment are different problems. Let \(W=\{q_1(\theta),\ldots,q_m(\theta)\}\) be the family of supported analytical request templates, defined independently of implementation. Parameters \(\theta\) are business values such as a site, product, interval, or threshold; they are not code or primitive lists.

At design time, engineers and domain experts author, review, test, and version a program for each supported template. At request time, deterministic policy maps an accepted canonical request \(c\) to one approved template and its allowed bindings:
\begin{equation}\Pi(c)=(j,\theta)\quad\text{or}\quad\text{clarify/reject}.\label{eq:policy}\end{equation}
The language model does not emit SQL, primitive names, primitive order, or a program identifier.
\begin{theorem}[Exact implementation of a supported policy family]\label{thm:policy}
If every template in \(W\) belongs to the analytical class of Theorem~\ref{thm:analytical}, then pre-written programs and deterministic policy can execute every correctly bound member of \(W\) exactly, without model-selected composition.
\end{theorem}
The result follows by compiling each symbolic template once and leaving only typed business parameters to bind at request time. A finite deployed catalog is intentionally narrower than the full algebra. If no policy rule covers a request, the system exposes a coverage gap rather than inventing a method.

This gives a useful separation:
\begin{quote}\textbf{The algebra answers ``can this analysis be expressed?'' Policy answers ``is this analysis approved and supported here?''}\end{quote}

\section{What is guaranteed - and what is not}
Fix an accepted canonical request \(c\) and the complete governed state
\begin{equation}\Omega=(D,A,M,\Pi,C,K,N,O),\label{eq:state}\end{equation}
containing the data snapshot \(D\), authorization \(A\), meaning definitions \(M\), policy \(\Pi\), program catalog \(C\), kernels \(K\), numerical/order conventions \(N\), and output contract \(O\).
\begin{theorem}[Evidence-preserving replay]\label{thm:replay}
If policy resolution and every selected kernel are deterministic, and ordering, numerical behavior, domain errors, and serialization follow fixed conventions, then repeated successful executions with the same \(c\) and \(\Omega\) return the same analytical contract and invariant execution certificate.
\end{theorem}
The proof is induction over the program dependency graph: the same sources and same deterministic node functions produce the same node outputs; canonical serialization then produces the same final representation. Operational telemetry such as latency and run identifiers is not part of analytical equality.

The guarantee begins after meaning has been accepted. A model can still map the user's words to the wrong meaning. The data can be incomplete. A registered policy can be substantively wrong. Determinism does not turn these into correct answers; it makes the chosen analysis replayable and inspectable.

Approximate kernels require an additional condition. If exact scores \(s_i\) are approximated by \(\widehat s_i\) with \(\max_i|\widehat s_i-s_i|\leq\eta\), threshold membership at \(\tau\) is unchanged whenever every score lies more than \(\eta\) from the threshold. Likewise, top-\(k\) membership is unchanged if the exact gap between positions \(k\) and \(k+1\) exceeds \(2\eta\). Without such margins, small numerical changes can alter the evidence population.

Finally, deterministic tools alone do not make a runtime-planning agent deterministic at the analytical level. If two valid deterministic procedures implement different functions and a model may choose either, component determinism does not guarantee that the same function is selected. Nor does tool availability alone imply that more planning steps will converge to the correct one. Those stronger claims require a specified search process, a progress condition, and a sound acceptance rule.

\section{Empirical demonstration: does runtime planning help?}
The theory above establishes representability and replay under stated assumptions. The experiment asks a narrower practical question: on analyses that are already supported by policy, what happens when the analytical procedure is instead created at runtime?
\subsection{Systems compared}
We used one synthetic governed quality/manufacturing dataset and eleven analytical tasks: counts, grouping including zero-count groups, absence queries, ranking with ties, rates, period change, contribution arithmetic, join multiplicity, exact means, score thresholds, and region grouping. Each task had an independent reference specification and evidence contract.

We compared four setups on the same hardware:
\begin{center}\footnotesize
\begin{tabularx}{\linewidth}{@{}>{\raggedright\arraybackslash}p{1.27in}>{\raggedright\arraybackslash}p{1.40in}Y@{}}
\toprule
\textbf{Setup}&\textbf{Model}&\textbf{How the method is chosen}\\\midrule
Runtime-planning agent&Qwen3-8B&Model inspects data, writes SQL, executes it, may revise, and selects the final procedure.\\
Runtime-planning agent&Ministral-3-8B&Same runtime procedure-creation protocol.\\
Runtime-planning agent&Granite-4.1-8B&Same runtime procedure-creation protocol.\\
Policy-executed analyzer&Qwen3-8B for language only&Qwen emits governed intent; deterministic policy executes the pre-written program.\\
\bottomrule
\end{tabularx}\end{center}
The Qwen comparison is particularly useful because the same model appears on both sides. The architectural difference is not model capability; it is whether the model owns the analytical method.

The benchmark used an NVIDIA RTX PRO 6000 Blackwell Server Edition, BF16 inference, temperature 0, one request at a time, and a tool budget of eight. Exact model revisions and the full harness are retained in the companion artifact.
\subsection{Two panels isolate two questions}
\textbf{Natural-language end to end.} Every setup starts from the same user question. Runtime-planning agents must both interpret the question and construct the SQL procedure. In the policy-executed analyzer, Qwen performs semantic interpretation only, then policy runs the method.

\textbf{Fixed canonical request.} Every setup receives the same already accepted meaning. Runtime-planning agents still construct a procedure. The policy-executed analyzer makes no model call and executes the approved program directly. This panel removes semantic interpretation as a possible explanation for analytical failure.

Each selected runtime procedure was frozen after development-snapshot planning and then evaluated without model assistance on four held-out database variants covering multiplicity/absence, time boundaries and scope, empty populations and zero denominators, and ties/scores/nulls. A procedure therefore had to implement the analytical function, not merely reproduce one visible number.

The primary metric was \textbf{exact contract accuracy}: the final procedure had to match the independent reference value, governed meaning, qualifying records, and evidence roles on all five snapshots. Returning the correct number with the wrong records did not count as exact.

\section{Results}
\subsection{The main result}
\callout{\textbf{In this specific setup, the runtime-planning agents did not perform successfully on the primary metric.} Across 330 agent episodes, 55 returned a final procedure and \textbf{0 of 330} produced an exact answer-and-evidence contract over all five snapshots. The policy-executed analyzer returned \textbf{110 exact contracts in 110 episodes}.}
\begin{table}[htbp]
\centering\scriptsize
\begin{tabularx}{\linewidth}{@{}lYrrrrr@{}}
\toprule
\textbf{Panel}&\textbf{Setup}&\textbf{Final}&\textbf{Exact}&\textbf{Tokens}&\textbf{Retries}&\textbf{Time (s)}\\\midrule
Natural language&Runtime agent: Qwen3-8B&0/55&0/55&6,850&0.64&6.372\\
Natural language&Runtime agent: Ministral-3-8B&10/55&0/55&11,219&0.00&20.029\\
Natural language&Runtime agent: Granite-4.1-8B&0/55&0/55&17,307&0.00&12.815\\
Natural language&\textbf{Policy analyzer: Qwen3-8B}&\textbf{55/55}&\textbf{55/55}&\textbf{1,433}&\textbf{0.00}&\textbf{0.219}\\\midrule
Fixed request&Runtime agent: Qwen3-8B&15/55&0/55&8,019&0.36&10.771\\
Fixed request&Runtime agent: Ministral-3-8B&5/55&0/55&15,252&0.09&23.670\\
Fixed request&Runtime agent: Granite-4.1-8B&25/55&0/55&17,912&2.91&34.003\\
Fixed request&\textbf{Policy execution}&\textbf{55/55}&\textbf{55/55}&\textbf{0}&\textbf{0.00}&\textbf{0.0028}\\
\bottomrule
\end{tabularx}
\caption{Completed 440-episode benchmark. ``Final'' means a final procedure/result was returned, not that it was correct. Tokens are total model input+output tokens. Retries are additional semantic or complete-procedure attempts recorded by the harness. Time is online wall time; model loading, warm-up, and offline held-out scoring are excluded.}\label{tab:results}
\end{table}
\subsection{How to read these numbers}
The strongest observation is not simply that the runtime agents used more tokens. It is that, under this protocol, they did not produce the required analytical contract. Of 330 runtime-planning episodes, 220 ended in rejection, 55 exhausted the tool budget, and 55 returned a final program. None of the returned programs was exact on the five-snapshot suite.

Five episodes did return the correct numerical answer on the visible development snapshot. All five were Ministral repetitions of the simplest count task. The generated query counted the right 17 records on that snapshot, but labeled evidence with the wrong role and used a grouped aggregate that returned no row for an empty population. It therefore matched the number but not the governed analysis. This is why the result contract includes supporting records and edge-case behavior.

The fixed-request panel is particularly informative. There, the intended meaning was supplied to every system, so semantic misunderstanding cannot explain the failures. The runtime agents still had to decide how to implement the analysis and produced \textbf{0 exact contracts in 165 episodes}. Policy execution returned \textbf{55/55} exactly with no model inference.

The same-model Qwen comparison isolates method ownership. In the natural-language panel, the policy analyzer used about \textbf{4.8 times fewer tokens} and \textbf{29 times less online time} than the Qwen runtime agent, while moving from 0/55 to 55/55 exact contracts. In the fixed-request panel, policy execution used \textbf{zero model tokens} and \textbf{0.0028 seconds} on average, compared with 8,019 tokens and 10.771 seconds for Qwen runtime planning.

Low retry counts should not be read as efficient success. Many episodes rejected or exhausted their budget before a correct procedure was selected. At temperature zero, several returned procedures were also repeatable across runs but wrong. Repeatability and correctness are different properties.

\section{Discussion}
The mathematical and empirical results answer different questions. The proofs show that restriction need not cause loss of expressiveness inside the declared analytical class, and that fixed governed state gives replayable result and evidence. The benchmark does not prove those theorems; it illustrates why method ownership can matter operationally.

The benchmark also does not prove that tool-using agents are generally incapable of analysis. A larger model, native function calling, different prompts, more budget, program synthesis with formal verification, or a sound search-and-acceptance procedure could perform better. Runtime planning is appropriate when finding the method is itself the task: open-ended research, method discovery, code generation, or operational planning. The claim here is narrower: for a supported factual analysis whose approved method is already known, asking a model to rediscover the method at request time adds a failure surface that is not needed for expressiveness.

The policy approach has its own cost. Approved programs must be authored, reviewed, tested, versioned, and maintained. Coverage grows deliberately rather than invisibly. The right empirical question for a product is therefore not ``can the algebra express everything?'' but ``what fraction of real user questions map to reviewed programs, and what does it cost to extend that coverage?'' A representative production workload is needed to answer that question.

\section{Conclusion}
A useful enterprise analyzer needs two kinds of flexibility in different places. Human language is flexible, so interpretation may benefit from a language model. Factual measurement should be stable, so the analytical method should be governed when the task is already known.

The central technical result is modest but strong: a compact relational core, extended with explicit analytical kernels, can exactly implement a broad stated class of finite data analyses. Deterministic policy can select pre-written implementations for supported requests without allowing a model to compose primitives at runtime. With fixed governed state, the result and evidence are replayable.

The experiment gives a concrete instance of the architectural trade-off. In this specific 440-episode setup, three runtime-planning agents produced no exact full-suite analytical contracts, while semantic interpretation followed by deterministic policy execution produced 110/110. The result should not be generalized beyond the tested configuration, but it supports a practical design principle:
\begin{center}\textbf{Let the model help determine what the user means. Let policy determine how an approved analysis is performed. Return the evidence with the answer.}\end{center}

\appendix
\section{Proof details}\label{app:proofs}
\subsection{Proof of Theorem \ref{thm:relational}}
For a finite set of named variables \(X\), define
\[
U_X=\prod_{x\in X}\rho_x(A_{\operatorname{sort}(x)}),\qquad U_\varnothing=\mathbf{1}.
\]
Alpha-rename bound variables so every quantifier introduces a fresh name. For \(X\supseteq\fv(\varphi)\), construct a relation \(T_X(\varphi)\) with schema \(X\). For a relational atom, rename source columns to fresh names, multiply the source by \(U_X\), select equalities linking source positions to the atom's variables or constants, and project to \(X\). Atomic comparisons are selections on \(U_X\). Boolean truth and falsehood are \(U_X\) and \(U_X\setminus U_X\). For compound formulas define
\begin{align*}
T_X(\neg\varphi)&=U_X\setminus T_X(\varphi),\\
T_X(\varphi\vee\psi)&=T_X(\varphi)\cup T_X(\psi),\\
T_X(\varphi\wedge\psi)&=T_X(\varphi)\setminus\bigl(T_X(\varphi)\setminus T_X(\psi)\bigr),\\
T_X(\exists y\,\varphi)&=\pi_XT_{X\cup\{y\}}(\varphi),\\
T_X(\forall y\,\varphi)&=U_X\setminus\pi_X\bigl(U_{X\cup\{y\}}\setminus T_{X\cup\{y\}}(\varphi)\bigr).
\end{align*}
We prove by structural induction that
\[
\bar a\in T_X(\varphi)\quad\Longleftrightarrow\quad\bar a\in U_X\text{ and }D,\bar a\models\varphi.
\]
The atom and comparison cases follow directly from the selections. Difference from \(U_X\) implements logical negation, union implements disjunction, and the displayed difference identity implements intersection/conjunction. Projection retains exactly those assignments having a witness for \(y\). For universal quantification, the inner difference identifies counterexamples, projection identifies assignments with at least one counterexample, and the outer difference removes them. This also gives vacuous truth when the candidate domain is empty. Setting \(X=\fv(\varphi)\) proves the forward translation.

For the reverse direction, induct on a relational expression. Base relations become relational atoms; selection adds a comparison; union becomes disjunction; difference becomes conjunction with negation of the right expression; product conjoins formulas over disjoint variable names; projection existentially quantifies removed attributes; and renaming renames variables. Every step preserves tuple membership.
\subsection{Typed closure and replay}
A finite acyclic program of terminating typed kernels terminates by induction over any topological order: each node is invoked only after its typed predecessors have terminated, and there are finitely many nodes. For replay, fixing \(c\) and \(\Omega\) fixes the selected graph, bindings, sources, and kernel state. Equal predecessor values therefore give equal node outputs or declared errors. Independent nodes communicate only through declared edges, so legal scheduling differences cannot alter the sinks. Canonical ordering and scalar encoding then serialize equal analytical outputs and invariant metadata identically.

\end{document}